\documentclass{ecai} 

\usepackage{graphicx}
\usepackage{amsfonts,amsmath,amsthm}
\usepackage[algo2e,ruled,vlined, linesnumbered]{algorithm2e}
\newtheorem{assume}{Assumption}
\newtheorem{theorem}{Theorem}
\newtheorem{lemma}{Lemma}

\newcommand{\BibTeX}{B\kern-.05em{\sc i\kern-.025em b}\kern-.08em\TeX}

\begin{document}


\begin{frontmatter}


\paperid{7746} 


\title{Enhancing Trust-Region Bayesian Optimization via Newton Methods}


\author[A]{\fnms{Quanlin}~\snm{Chen}\orcid{0009-0002-3789-7627}\footnote{Equal contribution.}}
\author[A]{\fnms{Yiyu}~\snm{Chen}\orcid{0000-0001-7186-7371}\footnotemark}
\author[A]{\fnms{Jing}~\snm{Huo}\orcid{0000-0002-8504-455X}\thanks{Corresponding Author. Email: huojing@nju.edu.cn}} 
\author[B]{\fnms{Tianyu}~\snm{Ding}\orcid{0000-0001-8445-4330}}
\author[A]{\fnms{Yang}~\snm{Gao}\orcid{0000-0002-2488-1813}}
\author[C]{\fnms{Yuetong}~\snm{Chen}\orcid{0009-0004-7044-3402}}

\address[A]{State Key Laboratory for Novel Software Technology, Nanjing University, Nanjing 210023, China}
\address[B]{Applied Sciences Group, Microsoft, Redmond, WA 98034, USA}
\address[C]{Sun Yat-sen University, Guangzhou 510275, China}


\begin{abstract}
    Bayesian Optimization (BO) has been widely applied
to optimize expensive black-box functions while retaining sample efficiency.
However, scaling BO to high-dimensional spaces remains challenging.
Existing literature proposes performing standard BO in multiple local trust regions (TuRBO)
for heterogeneous modeling of the objective function and avoiding over-exploration.
Despite its advantages, using local Gaussian Processes (GPs) reduces sampling efficiency compared to a global GP.
To enhance sampling efficiency while preserving heterogeneous modeling, we propose to construct multiple local quadratic models using gradients and Hessians from a global GP,
and select new sample points by solving the bound-constrained quadratic program.
Additionally, we address the issue of vanishing gradients of GPs in high-dimensional spaces.
We provide a convergence analysis and demonstrate through experimental results that
our method enhances the efficacy of TuRBO and outperforms a wide range of high-dimensional BO techniques
on synthetic functions and real-world applications.
\end{abstract}

\end{frontmatter}

\section{Introduction}
\label{sec:intro}
Bayesian Optimization (BO) has been one of the popular methods for the global optimization of expensive black-box functions
due to its high sampling efficiency.
Applications include hyperparameter tuning for deep learning \cite{DBLP:conf/iclr/HvarfnerSSLHN22},
discovering new molecules for chemical engineering \cite{vaebo}, 
searching an optimal policy for reinforcement learning \cite{gibo}, and so on.
BO is a sequential model-based approach consisting of two main components: a surrogate model and an acquisition function.
The surrogate model, typically implemented as a Gaussian Process regression, is used to improve the sampling efficiency of BO by modeling the objective function.
The acquisition function is used to determine the next sample point.

While BO performs well in optimizing low-dimensional functions,
it struggles with high-dimensional problems for several reasons.
First, the surrogate model loses accuracy in the high-dimensional space when estimating the objective function.
This is because it is impossible to fill the high-dimensional space with finite sample points, 
even with a large sample size \cite{gyorfi2002distribution}.
Second, the computational complexity of optimizing the acquisition function grows exponentially with dimensions \cite{add-bo}.

Various methods have been proposed to address the curses of dimensionality in BO.
The vast majority of the prior work assumes special structures in the objective function,
such as additive structure \cite{add-bo,add-bo-tree} or intrinsic dimension \cite{rembo,alebo}.
However, these assumptions are often too restrictive for widespread application.
Other works directly improve the high-dimensional BO without additional assumptions,
including TuRBO \cite{turbo}, GIBO \cite{gibo}, and MPD \cite{MPD}.

This paper focuses on trust-region BO (TuRBO).
TuRBO is attractive because it uses local GPs for heterogeneous modeling of the objective function
and performs BO locally in multiple trust regions to avoid over-exploration.
However, using local GPs reduces sampling efficiency compared to a global GP.
To overcome this limitation, we propose a new trust-region BO method (Newton-BO) that incorporates Newton methods into BO.
It constructs multiple local quadratic models using gradients and Hessians derived from a global GP,
enabling heterogeneous modeling of the objective function while maintaining the same sample efficiency of a global GP.
In each iteration, our method consists of three main stages:
building multiple local quadratic models using derivatives from a global GP,
selecting new sample points by solving the bound-constrained quadratic program in each trust region,
and adjusting radii of trust regions based on new evaluations.
To ensure global optimization, the algorithm restarts in a new location selected using Predictive Entropy Search.
Additionally, we address the issue of vanishing gradients of GPs in high-dimensional spaces and
provide theoretical proof that our method converges to stationary points with high probability.
In summary, our main contributions are:
\begin{itemize}
    \item Proposing a new trust-region BO method that incorporates Newton methods to enhance sampling efficiency while retaining heterogeneous modeling.
    \item Addressing the issue of vanishing gradients of GPs in high-dimensional spaces.
    \item Providing a convergence analysis guaranteeing the convergence of our proposed method.
    \item Empirically validating our method on synthetic and real-world applications, demonstrating improved efficacy over TuRBO and outperforming various high-dimensional BO methods.
\end{itemize}

\section{Related Work}
\label{sec:related}

In the realm of high-dimensional BO, there are generally three kinds of methods.
The first kind of method assumes the existence of a lower-dimensional structure within objective functions,
typically employing a three-stage process: 
producing a low-dimensional embedding, 
performing standard BO in this low-dimensional space, and
projecting found optimal points back to the original space. 
In REMBO \cite{rembo}, the low-dimensional embedding is achieved by using a random projection matrix.
But REMBO often produces points that fall outside the box bounds of the original space,
necessitating their projection onto the facet of the box and resulting in a harmful distortion.
Subsequently, several techniques are proposed to fix this problem \cite{alebo,rembo2020}.
In addition, the random low-dimensional embedding can be also achieved by randomized hashing functions \cite{hesbo,baxus}.
The key advantage of the hashing functions lies in their ability to effortlessly map candidate points back to the original space,
thus circumventing the need for clipping to box-bound facets.
Some works achieve linear embeddings based on learning.
For example, SIR-BO employs Sliced Inverse Regression to derive the linear embeddings, while
SI-BO \cite{si-bo} learns the linear embeddings via low-rank matrix recovery.
\cite{GarnettOH14} learn the linear embeddings by maximizing the marginal likelihood of GPs.
Besides, nonlinear embedding techniques have also been explored, particularly those based on Variational Autoencoders 
\cite{vaebo}.
However, these approaches typically require a substantially larger sample size.
In addition to embedding techniques, some research has focused on variable selection methods \cite{linebo,dropout,vs-bo,mcts-bo}.

The second kind of method assumes the existence of an additive structure for the objective function.
The additive objective function can be modeled by additive GPs \cite{add-bo},
allowing for more efficient maximization of the acquisition function.
However, the true additive structure still remains challenging to learn.
Several works propose to learn the underlying additive structure from training data.
For example, \cite{add-bo-skl} proposed a method that employs the Dirichlet process to assign input variables into distinct groups.
\cite{add-bo-graph} employ a dependency graph to model the interactions between input variables, 
allowing for the assignment of input variables into overlapping groups.
\cite{add-bo-tree} proposed a refinement that restricts the dependency graph to a tree structure,
reducing the computational complexity of maximizing acquisition functions.
In contrast to data-driven decomposition methods, RDUCB \cite{rducb} learns a random tree-based decomposition
to mitigate the potential mismatch between the objective function and additive GPs.

The third kind of method focuses on direct enhancements to the BO process in high-dimensional spaces,
without relying on any other assumption.
For example, TuRBO \cite{turbo}, GIBO \cite{gibo} and MPD \cite{MPD} adopt local strategies for BO to avoid over-exploration in high-dimensional spaces.
Another set of approaches focuses on partitioning the search space and identifying a promising region to perform BO more efficiently
\cite{lamtcs}.
Researchers have also proposed better initialization methods for optimizing high-dimensional acquisition functions efficiently \cite{elasticGP,aibo}.

GIBO and MPD are similar to ours, which also utilize gradients of GPs.
In contrast to their work, our work incorporates both gradient and Hessian information from GPs
and considers addressing the issue of vanishing gradients of GPs in high-dimensional spaces.

\section{Background}
\label{sec:background}
\subsection{Bayesian Optimization}
Bayesian optimization considers an optimization problem
$\min_{\mathbf{x}\in\mathcal{X}}f(\mathbf{x})$ where $f$ is a black-box and derivative-free function
over a hyper-rectangular feasible set $\mathcal{X}$.
As a sequential model-based approach, BO comprises two main components: 
a surrogate model and an acquisition function.
The surrogate model approximates the objective function, while 
the acquisition function, based on this model, determines the next sampling point.
Gaussian Process regression is typically employed as the surrogate model \cite{gp},
$f \sim \mathcal{GP}(m(\cdot), k(\cdot, \cdot))$ with a mean function
$m(\cdot)$ and a kernel $k(\cdot, \cdot)$. 
More specifically,
GP assumes that evaluations of any finite number sampling point $\mathbf{x}_{1:n}$ follow a joint Gaussian distribution,
$\mathbf{f}\sim \mathcal{N}(\mathbf{m}(\mathbf{x}_{1:n}),\mathbf{K}(\mathbf{x}_{1:n},\mathbf{x}_{1:n}))$.
For simplicity, $\mathbf{m}(\mathbf{x}_{1:n})$ and $\mathbf{K}(\mathbf{x}_{1:n},\mathbf{x}_{1:n})$ are abbreviated as $\mathbf{m}$ and $\mathbf{K}$, respectively.
Given training data $\mathcal{D}_n=\{\mathbf{x}_{1:n},\mathbf{y}_{1:n}\}$
and a new point $\mathbf{x}_*$, the joint distribution is given by
$$ \begin{bmatrix}
    \mathbf{y}_{1:n} \\ f(\mathbf{x}_*)
\end{bmatrix} \sim \mathcal{N}\left(\begin{bmatrix}
    \mathbf{m} \\ m(\mathbf{x}_*)
\end{bmatrix}, \begin{bmatrix}
    \mathbf{K}+\sigma_n^2\mathbf{I} & 
    \mathbf{k}(\mathbf{x}_{1:n},\mathbf{x}_*) \\ 
    \mathbf{k}(\mathbf{x}_*,\mathbf{x}_{1:n}) & 
    k(\mathbf{x}_*,\mathbf{x}_*)
\end{bmatrix}\right)$$
where $\sigma_n^2$ is the variance of Gaussian noise added to the observations. 
It follows from the Sherman-Morrison-Woodbury formula that
the posterior normal distribution for $f(\mathbf{x}_*)$
is given by $f(\mathbf{x}_*) | \mathcal{D}_n,\mathbf{x}_* \sim \mathcal{N}(\mu_n(\mathbf{x}_*),\sigma_n^2(\mathbf{x}_*))$ 
where
\begin{align*}
    \mu_n(\mathbf{x}_*) ={}& m(\mathbf{x}_*) + \mathbf{k}(\mathbf{x}_*,\mathbf{x}_{1:n})
    (\mathbf{K}+\sigma_n^2\mathbf{I})^{-1} (\mathbf{y}_{1:n}-\mathbf{m}) \\
    \sigma_n^2(\mathbf{x}_*) ={}& k(\mathbf{x}_*,\mathbf{x}_*) - \mathbf{k}(\mathbf{x}_*,\mathbf{x}_{1:n}) 
    (\mathbf{K}+\sigma_n^2\mathbf{I})^{-1} \mathbf{k}(\mathbf{x}_{1:n},\mathbf{x}_*)
\end{align*}

Based on this posterior, an acquisition function $\alpha(\cdot)$ is constructed to quantify the utility of sampling points.
Common choices include Expected Improvement (EI) \cite{ego}, Upper Confidence Bound (UCB) \cite{gp-ucb}, and Predictive Entropy Search (PES) \cite{PES}.
For example, PES evaluates the mutual information between the global minimizer $\mathbf{x}_*$ and predicted value $y$ given $\mathbf{x}$, i.e.,
\begin{align*}
    \alpha_n(\mathbf{x}) :={}& I(\mathbf{x}_*;y | \mathcal{D}_n,\mathbf{x}) \\
    ={}& H[p(y|\mathcal{D}_n,\mathbf{x})] - \mathop{\mathbb{E}}_{p(\mathbf{x}_*|\mathcal{D}_n)}[H[p(y|\mathcal{D}_n, \mathbf{x}, \mathbf{x}_*)]].
\end{align*}
The next sample point is determined by maximizing the acquisition function, $\mathbf{x}_{n+1} = \mathop{\arg\max}_{\mathbf{x}\in\mathcal{X}}\alpha_n(\mathbf{x})$.
After evaluating the objective function at $\mathbf{x}_{n+1}$, the process advances to the next iteration.

\subsection{Uniform Error Bounds of the GP}
Under the mild assumption of Lipschitz continuity for the kernel function, a directly computable probabilistic uniform error bound can be established.
\begin{assume}\label{a:gp-bound}
    The unknown objective function $f$ is a sample from a Gaussian process $\mathcal{GP}(0, k(\mathbf{x}, \mathbf{x}'))$
and observations are perturbed by Gaussian noise, $y = f(\mathbf{x}) + \epsilon$, where $\epsilon\sim\mathcal{N}(0,\sigma^2)$.
The kernel $k$ is Lipschitz continuous with the Lipschitz constant $L_k$ and possesses continuous partial derivatives up to the fourth order.
\end{assume}

According to Theorem 3.2 in \cite{gp-bound}, an unknown function $f$ satisfying Assumption \ref{a:gp-bound} is almost surely continuous on $\mathcal{X}$ and is Lipschitz continuous with high probability.
Furthermore, if $f$ has a Lipschitz constant $L_f$, then
the posterior mean function $\mu_t$ of the GP fitted on the training data $\mathcal{D}_t$ is continuous with a Lipschitz constant $L_{\mu_t}$,
and the standard deviation $\sigma_t$ admits a modulus of continuity $\omega_{\sigma_t}$ on $\mathcal{X}$, where
\begin{align*}
    L_{\mu_t} \leq{}& L_k \sqrt{t} \|(\mathbf{K}+\sigma_t^2\mathbf{I})^{-1}\mathbf{y}\|_2\\
    \omega_{\sigma_t}(\tau) \leq{}& \sqrt{2\tau L_k\left(1+t\|(\mathbf{K}+\sigma_t^2\mathbf{I})^{-1}\|_2
    \max_{\mathbf{x},\mathbf{x}'\in\mathcal{X}}k(\mathbf{x},\mathbf{x}')\right)}.
\end{align*}
Moreover, given $\delta \in (0, 1),~ \tau>0$, one has that
\begin{equation}\label{eq:gp-bound}
\mathbb{P}\left( |f(\mathbf{x}) - \mu_t(\mathbf{x})| 
\leq \sqrt{\beta(\tau)}\sigma_t(\mathbf{x}) + \gamma(\tau),~\forall \mathbf{x}\in\mathcal{X} \right) \geq 1-\delta,
\end{equation}
where
\begin{equation}\label{eq:gp-Lipschitz}
\begin{aligned}
    \beta(\tau) ={}& 2\log\left(\frac{M(\tau, \mathcal{X})}{\delta}\right),\\
    \gamma(\tau) ={}&  (L_{\mu_t} + L_f)\tau + \sqrt{\beta(\tau)}\omega_{\sigma_t}(\tau),
\end{aligned}
\end{equation}
and $M(\tau, \mathcal{X})$ is the covering number that is the minimum number of spherical balls
with radius $\tau$ required to completely cover $\mathcal{X}$.

\subsection{Trust-Region Bayesian Optimization (TuRBO)}
TuRBO abandons global surrogate modeling in favor of maintaining multiple trust regions simultaneously, enabling heterogeneous modeling and avoiding over-exploration.
Each trust region $\text{TR}^{(\ell)}$ is a hyperrectangle that employs an independent local GP model, $f^{(\ell)}\sim\mathcal{GP}^{(\ell)}(\mu^{(\ell)}(\mathbf{x}), k^{(\ell)}(\mathbf{x})).$
Then a batch of $q$ candidates is drawn from the union of all trust regions using Thompson sampling \cite{TS}, i.e.,
\begin{align*}
    \mathbf{x}_i ={}& \mathop{\arg\min}_\ell \mathop{\arg\min}_{\mathbf{x}\in \text{TR}^{(\ell)}} f_i^{(\ell)}\\
    &\text{ where } f_i^{(\ell)} \text{ is a sample from } \mathcal{GP}^{(\ell)}(\mu^{(\ell)}, k^{(\ell)})
\end{align*}
for $i=1,\dots,q$.
The trust-region radius of each TR will be updated adaptively based on new evaluations.
It is decreased if the optimizer appears stuck and increased if the optimizer finds better solutions.

\begin{algorithm2e}[!tb]
    \KwIn{Initial size $n$; The number of iterations $T$; The batch size $q$}
    \KwOut{The sample points and their evaluations $\mathcal{D}_T$}
    $\mathcal{D}_0=\{\mathbf{x}_{1:n}, \mathbf{y}_{1:n}\} \leftarrow$ Randomly sample $n$ points from the feasible set $\mathcal{X}$ and then evaluate these points\;
    \textbf{Initializations.} Choose an initial radius for each trust region, $\{\Delta_0^{(\ell)}\}_{\ell=1}^q$,
    and determine an initial point for each trust region, $\{\mathbf{x}^{(\ell)}\}_{\ell=1}^q\subset\mathcal{D}_0$\;
    \For{$k\leftarrow 1$ \KwTo $T$}{
        Build a global D-scaled GP based on the training data $\mathcal{D}_k$\;
        \For{$\ell\leftarrow 1$ \KwTo $q$}{
        Build a local quadratic model in the $\ell$-th trust region\;
        $\mathbf{x}_*^{(\ell)} \leftarrow$ Select a candidate by minimizing the model within the $\ell$-th trust region according to Eq.\ref{eq:quadratic-program}\;
        Evaluate the candidate, $y^{(\ell)}\leftarrow f(\mathbf{x}_*^{(\ell)})$\;
        Update the trust-region radius $\Delta_k^{(\ell)}$ based on new evaluations\;
        }
        Update the training data, $\mathcal{D}_{k+1}\leftarrow\mathcal{D}_k\cup \{\mathbf{x}_*^{(\ell)}, y^{(\ell)}\}_{\ell=1}^q$\;
    }
    \Return{$\mathcal{D}_T$}
\caption{Newton-BO}
\label{algo:turboD}
\end{algorithm2e}

\begin{figure}[t]
    \centering
    \includegraphics[width=0.49\linewidth]{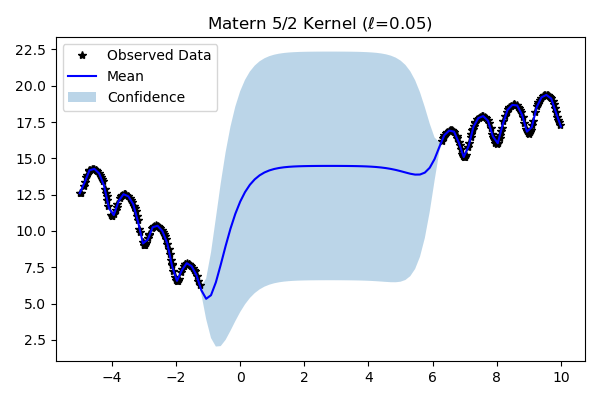}
    \includegraphics[width=0.49\linewidth]{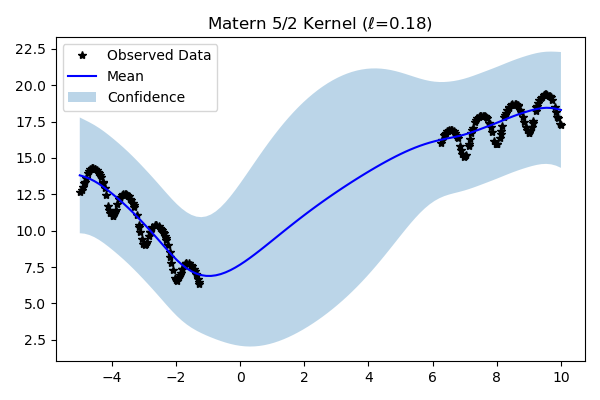}
\caption{
    Two GP models with varying length-scales attempting to model the Ackley function.
}
    \label{fig:matern}
\end{figure}

\begin{figure*}[t]
    \centering
    \includegraphics[width=\linewidth]{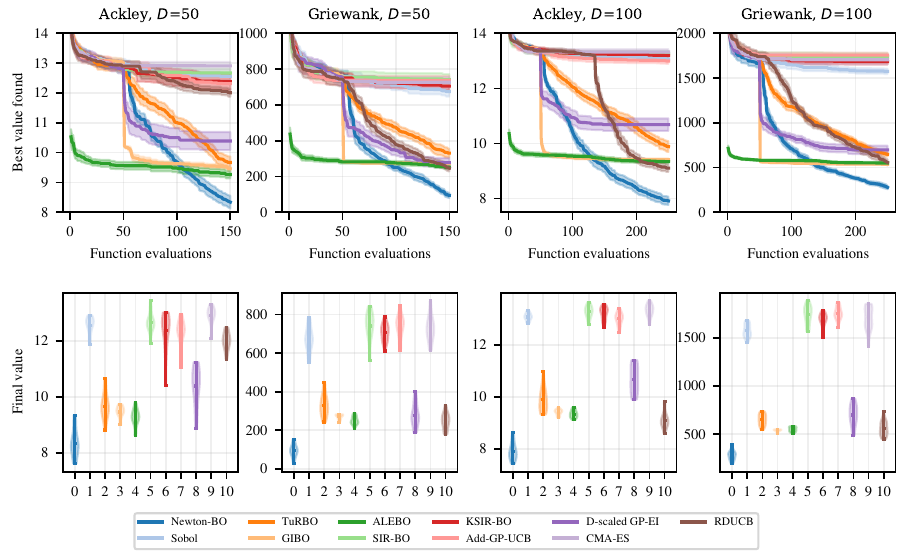}
\caption{We compare Newton-BO to baseline methods on 50-dimensional functions and 100-dimensional functions,
showing (Top row) optimal values by each iteration averaged over 20 repeated runs,
and (Bottom row) the distribution over the final optimal values over 20 repeated runs.
}
    \label{fig:full}
\end{figure*}

\begin{figure*}[t]
    \centering
    \includegraphics[width=\linewidth]{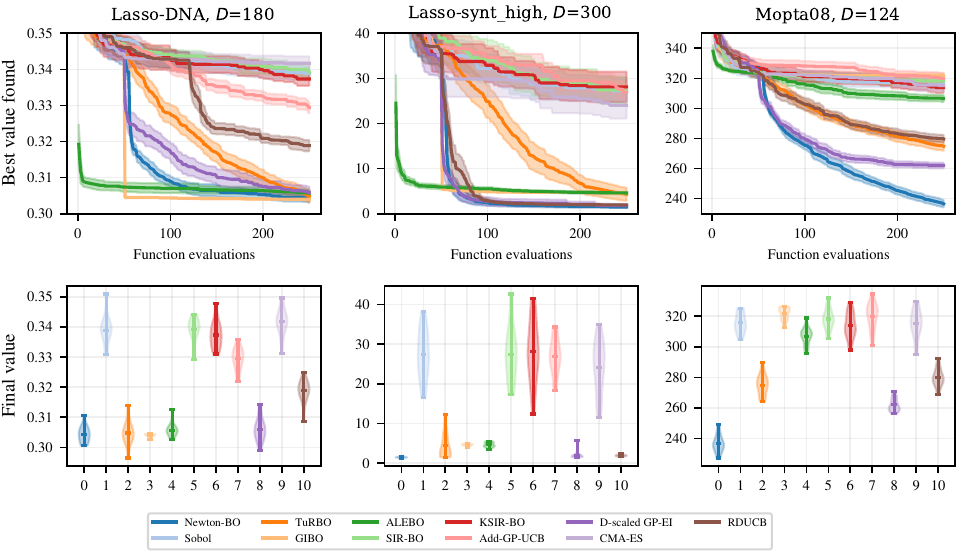}
\caption{
    We compare Newton-BO to baseline methods on the Lasso-DNA tuning ($D=180$),
Lasso-synt\_high tuning ($D=300$) and MOPTA vehicle design ($D=124$),
showing (Top row) optimal values by each iteration averaged over 20 repeated runs,
and (Bottom row) the distribution over the final optimal values over 20 repeated runs.
}
    \label{fig:real}
\end{figure*}

\begin{figure*}[t]
    \centering
    \includegraphics[width=\linewidth]{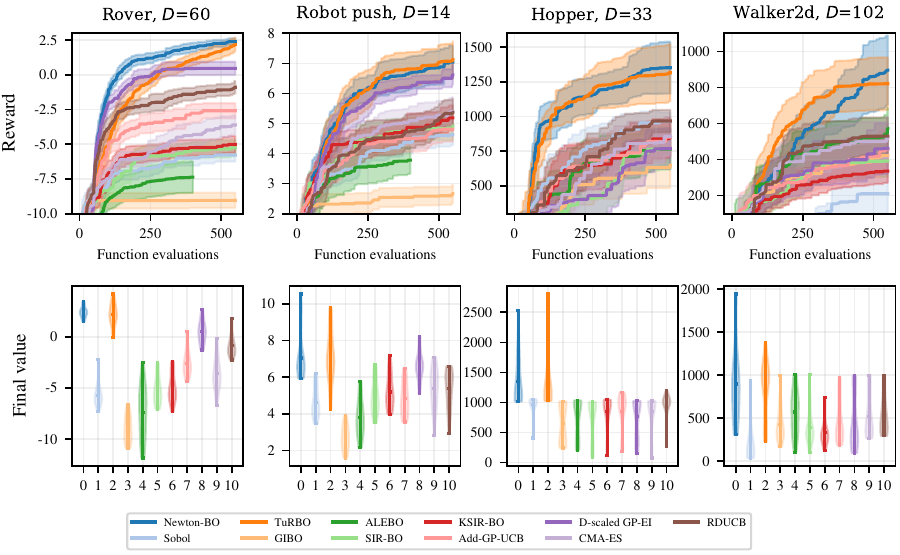}
\caption{
    We compare Newton-BO to baseline methods on the Rover trajectory planning, Robot pushing, Hopper, and Walker2d problems,
showing (Top row) optimal values by each iteration averaged over 20 repeated runs,
and (Bottom row) the distribution over the final optimal values over 20 repeated runs.
}
    \label{fig:motion}
\end{figure*}

\begin{figure*}[t]
    \centering
    \includegraphics[width=\linewidth]{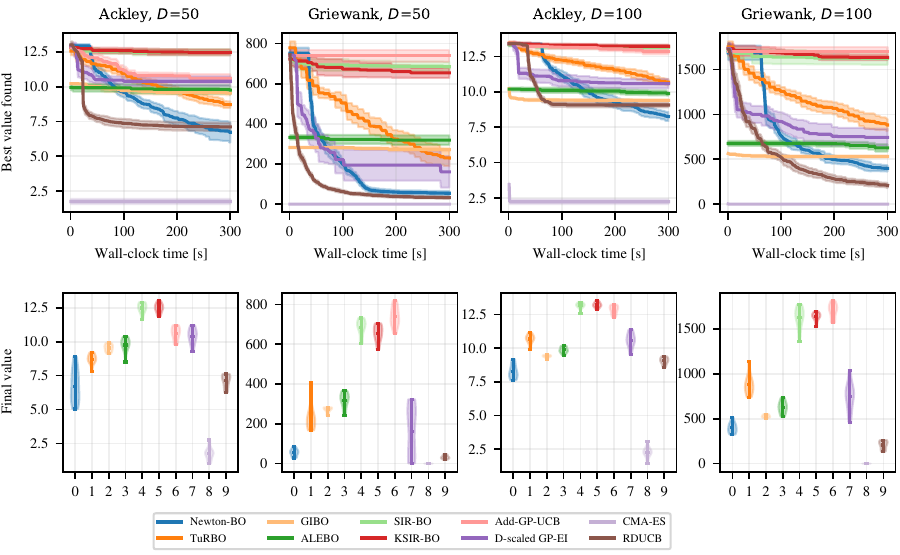}
\caption{
    Comparison between Newton-BO and baseline methods on 50-dimensional and 100-dimensional synthetic functions as a function of the wall-clock time [s]. 
}
    \label{fig:time}
\end{figure*}

\section{Method}
\label{sec:method}
In this section, we  introduce a novel trust-region BO method for optimizing high-dimensional black-box functions.
As discussed previously, TuRBO maintains multiple independent local GPs, resulting in reduced sampling efficiency compared to a global GP.
To enhance sampling efficiency, we construct multiple  local quadratic models using gradients and Hessians from a global GP.
This approach allows for heterogeneous modeling of the objective function while maintaining the same sample efficiency of a global GP.
Furthermore, to achieve global optimization, the algorithm restarts in a new location selected by Predictive Entropy Search.
Additionally, we address the issue of vanishing gradients of GPs in high-dimensional spaces.

\paragraph{Local modeling.}
Instead of just evaluating a single point in each iteration, we consider a batch-sequential extension which enables parallel evaluation of multiple points.
Specifically, given $q$ points $\{\mathbf{x}^{(\ell)}\}_{\ell=1}^q$,
$q$ local quadratic models are constructed around these points, defined as
\begin{equation*}
    m^{(\ell)}(\mathbf{x}^{(\ell)}+\mathbf{s}) = f(\mathbf{x}^{(\ell)}) + \mathbf{s}^\top\mathbf{g}^{(\ell)} + \frac12\mathbf{s}^\top\mathbf{B}^{(\ell)}\mathbf{s},\ \ell=1,\dots,q
\end{equation*}
where $\mathbf{g}$ and $\mathbf{B}$ approximate the gradient and Hessian of the objective function, respectively.
Since the derivatives of the objective function $f$ are unknown,  they are estimated using derivatives of GPs with the reparameterization trick, i.e.,
\begin{align*}
\mathbf{g}^{(\ell)} ={}& \nabla \mu(\mathbf{x}^{(\ell)}) + \lambda\nabla \sigma(\mathbf{x}^{(\ell)}),\\
\mathbf{B}^{(\ell)} ={}& \nabla^2\mu(\mathbf{x}^{(\ell)})+\lambda\nabla^2\sigma(\mathbf{x}^{(\ell)}),
\end{align*} 
where $\mu(\cdot)$ and $\sigma(\cdot)$ are the posterior mean and standard deviation of the GP model, respectively.
Instead of sampled from a standard normal distribution,  the variable
$\lambda$ is sampled from a truncated standard normal distribution with truncation range $(-1,1)$ to ensure local convergence (see Appendix \ref{sec:proof} for details).
It is important to note that all derivatives are derived from a single global GP, resulting in higher sample efficiency compared with local GPs.

\paragraph{Trust regions.}
To ensure the quadratic model accurately approximates $f$,
$\mathbf{x}^{(\ell)}+\mathbf{s}$ needs to be restricted to a trust region $\mathcal{B}^{(\ell)}$ defined as
$$ \mathcal{B}^{(\ell)}:=\{\mathbf{x}\in\mathbb{R}^D\mid \|\mathbf{x}-\mathbf{x}^{(\ell)}\|\leq\Delta^{(\ell)}\},$$
where $\Delta^{(\ell)}$ is the trust-region radius, adjusted iteratively.
In BO, the search space is typically a rectangular box.
Without loss of generality, we assume that the box is $[0,1]^D$.
Given this constraint, the trust region becomes
$$  \mathcal{B}^{(\ell)}:=\{\mathbf{x}\mid \|\mathbf{x}-\mathbf{x}^{(\ell)}\|\leq\Delta^{(\ell)},~\mathbf{0}\leq \mathbf{x} \leq \mathbf{1}\}. $$
To simplify this, we adopt the $\infty$-norm for the trust region, which transforms $\mathcal{B}^{(\ell)}$ into a simple rectangular,
$$ \mathcal{B}^{(\ell)}:=\{\mathbf{x}\mid -\Delta^{(\ell)}\mathbf{1}\leq\mathbf{x}-\mathbf{x}^{(\ell)}\leq\Delta^{(\ell)}\mathbf{1},~\mathbf{0}\leq \mathbf{x} \leq \mathbf{1}\}. $$
Candidates are selected by solving the bound-constrained quadratic program,
\begin{equation} \label{eq:quadratic-program}
    \underset{\mathbf{s}}{\operatorname{minimize}} \quad  m(\mathbf{x}^{(\ell)}+\mathbf{s}),\quad \operatorname{subject\ to}~ \mathbf{x}^{(\ell)}+\mathbf{s} \in \mathcal{B}^{(\ell)}.
\end{equation}
This is another difference between our method and TuRBO which obtains candidates through Thompson sampling.
The above problem is typically solved using gradient projection methods.
However, the Hessian of the GP is often nearly singular, posing challenges for gradient projection techniques.
Such methods may require numerous iterations and yield only marginal improvements in each step.
Instead, we utilize the Covariance Matrix Adaptation Evolution Strategy (CMA-ES) \cite{cma-es} to solve this bound-constrained quadratic program.

\paragraph{Restarting strategy}
In this paper, we adopt a radius update strategy similar to that used in TuRBO, which has demonstrated effectiveness in balancing local exploitation with global exploration.

Additionally, our method restarts in two specific scenarios: (i) when the radius falls below a predetermined minimum threshold $\Delta_{\operatorname{min}}$; and
(ii) when the norm of $\mathbf{g}$ is less than $10^{-5}$.
To achieve global optimization, the algorithm restarts at new locations selected by an acquisition function.
We employ parallel PES \cite{PPES} as the acquisition function, as it is one of the most popular non-greedy strategies in batch BO.

\paragraph{Challenges in the high-dimensional space.}
\cite{elasticGP} claims that in high dimensions the acquisition function becomes very flat across large regions, causing the gradient of the acquisition function to vanish.
Since the GP surface can be viewed  as a special case of the UCB acquisition function, our method also encounters  the challenge.

Before discussing potential solutions, we first illustrate this issue empirically.
To simulate the sparsity of high-dimensional data, we generate uniform samples only in the intervals $[-5,0]$ and $[5,10]$ of the Ackley function, building GPs with varying  length-scales.
As shown in Figure \ref{fig:matern}, a GP with a small length-scale is sharp in areas with abundant data but flat in areas with sparse data.
Conversely, a GP with a large length-scale exhibits a significant gradient even in regions with limited data.

Lemma 1 in \cite{elasticGP} indicates  that the gradient of the acquisition function becomes significant when the length-scale is sufficiently large, which is also applicable to the surface of GPs.
Instead of manually adjusting length-scales, we employ  D-scaled GPs \cite{D-scaled-GP} which utilize a LogNormal prior for the length-scales,
$$\ell_i\sim\mathcal{LN}\left(-4+\frac{\log D}{2},1\right).$$
This prior ensures that the length-scales do not become too small. 
We denote our method as Newton-BO, as presented in Algorithm \ref{algo:turboD}.

\section{A convergence analysis}
\label{sec:analysis}
Our method shares several key features with trust-region derivative-free optimization methods,
including the use of quadratic models to approximate the objective function and adaptive trust region updates.
However, a crucial distinction lies in the nature of the error between the quadratic model and the objective function.
This error is probabilistic in our approach, while 
it is typically deterministic in derivative-free optimization methods using interpolation techniques.
This probabilistic aspect necessitates a verification of the coherence between
the derivatives of GPs and those of the objective function.
This fundamental difference precludes the direct application of standard convergence theory for derivative-free methods to our method.
Consequently, we must reconsider the convergence analysis in detail.

To maintain analytical simplicity,
we adopt the same assumptions as \cite{conn1997convergence}
and follow their trust region update strategy, as outlined in Appendix \ref{sec:proof}.
\begin{assume}\label{a:f-gh-bound}
    The objective function $f:\mathbb{R}^D\to\mathbb{R}$ is twice continuously differentiable whose gradient $\nabla f(\mathbf{x})$ and Hessian
    $\nabla^2 f(\mathbf{x})$ is uniformly bounded in the norm. In other words, there are constants $\kappa_{fg}>0$ and $\kappa_{fh}>0$
    such that 
    $$\|\nabla f(\mathbf{x})\|_2 \leq \kappa_{fg},\quad \|\nabla^2 f(\mathbf{x})\|_2 < \kappa_{fh}$$ 
    for all $\mathbf{x}\in\mathbb{R}^D$.
\end{assume}
\begin{assume}\label{a:f-bounded-below}
    The objective function is bounded below on $\mathbb{R}^D$.
\end{assume}
\begin{assume}\label{a:mh-bound}
    The approximate Hessians $\mathbf{B}_k$ are uniformly bounded in the norm. In other words, there is a constant $\kappa_{mh}>0$ such that
    $\|\mathbf{B}_k\|_2 \leq \kappa_{mh},~\forall \mathbf{x}\in\mathcal{B}_k$.
\end{assume}

The following theorem shows that our approach will converge to stationary points with high probability.
The proof is inspired by that of \cite{conn1997convergence} and provided in Appendix \ref{sec:proof}.

\begin{theorem}\label{th:converge}
    Assume that Assumption \ref{a:gp-bound}-\ref{a:mh-bound} hold. Then given $\delta\in(0,1)$, there is a sequence of iterations $\{\mathbf{x}_k\}$ such that
    $\forall \epsilon\in(0,1)~\exists N$,
\begin{equation*}
    \mathbb{P}\left( \inf_{k>N}\|\nabla f(\mathbf{x}_k)\| = 0 \right) \geq 1-\delta.
\end{equation*}
\end{theorem}

\section{Experimental Results}
\label{sec:experiments}

In this section, we evaluate our method (Newton-BO) on a wide range of benchmarks:
synthetic functions, Lasso tuning problems, a vehicle design problem, a robot pushing problem, a rover trajectory planning problem, and MuJoCo locomotion tasks.

We compare our method (Newton-BO) to a broad set of baseline methods:
linear embedding (ALEBO \cite{alebo}, SIR-BO), 
nonlinear embedding (KSIR-BO \cite{sir-bo}),
additive models (Add-GP-UCB \cite{add-bo}, RDUCB),
local-search methods (TuRBO, GIBO),
D-scaled GP-EI \cite{D-scaled-GP},
quasirandom search (Sobol),
and an evolutionary strategy (CMA-ES).
For BO using embedding, we take $d=10$ for these experiments. 
For Add-GP-UCB, we take $d=4$ for each group.
Newton-BO and TuRBO maintain 5 trust regions simultaneously.
We test all methods using 50 initial points and batch size of $q=5$.
More details about the settings can be found in Appendix \ref{sec:settings}.
Code to reproduce the results of this paper is available at \url{https://github.com/qlchen2117/NewtonBO}.

\subsection{Synthetic Experiments}
First, we consider the 50-dimensional Ackley function in the domain $[-5,10]^{50}$,
and the 50-dimensional Griewank function in the domain $[-300,600]^{50}$.
Both functions feature numerous local minima and a global minimum,
making them suitable for testing global optimization methods.
Figure \ref{fig:full} shows that Newton-BO enhances the efficacy
of TuRBO and gets the best performance of all methods on the mid-dimensional synthetic functions.
The initialization strategy of ALEBO favors sampling points away from the boundary, resulting in high-quality initial samples.
However, the optimizer of ALEBO tends to stagnate when objective functions lack lower-dimensional structure.
SIR-BO and KSIR-BO demonstrate poor performance in this problem, yielding results comparable to random search.
Add-GP-UCB also underperforms on this problem because objective functions lack additive structure.

Second, we consider the 100-dimensional Ackley function in the domain $[-5,10]^{100}$,
and the 100-dimensional Griewank function in the domain $[-300,600]^{100}$.
Figure \ref{fig:full} shows that Newton-BO again enhances the efficacy of TuRBO and gets the best performance among all methods on the high-dimensional synthetic functions.
GIBO always samples the midpoint of the domain after initialization.
It suffers from the vanishing gradients of GPs in the high-dimensional spaces, causing it to become stuck at the midpoint.
RDUCB performs well on both Griewank functions and 100-dimensional Ackley functions, despite they lacking the additive structure.

\subsection{Real-World Problems}
\paragraph{Weighted Lasso Tuning.} 
We consider the problem of tuning the Lasso (Least Absolute Shrinkage and Selection Operator) regression models.
LassoBench \cite{lassobench} provides a set of benchmark problems for tuning penalty terms for Lasso models.
In Lasso, each regression coefficient corresponds to a penalty term, so the number of hyperparameters equals the number of features in the dataset. 
We focus on two Lasso tuning problems: a 180-dimensional DNA dataset with 43 effective dimensions,
and a 300-dimensional synthetic dataset with 15 effective dimensions.

Figure \ref{fig:real} shows that
Newton-BO performs comparably to D-scaled GP-EI while outperforming other methods on the Lasso-synt\_high problem.
For the Lasso-DNA problem, Newton-BO eventually attains optimal values comparable to GIBO while outperforming other methods.
GIBO, after initially sampling the midpoint, stagnates due to vanishing gradients of GPs in high-dimensional spaces.
Its performance is primarily attributed to this initial midpoint sampling.
ALEBO also becomes stuck after initialization, despite the existence of lower-dimensional structure in these problems.
RDUCB performs well on the Lasso-synt\_high but but struggles with Lasso-DNA.

\paragraph{Vehicle Design.}
We consider the vehicle design problem with a soft penalty as defined in \cite{saasbo}.
The objective is to minimize the mass of a vehicle characterized by 124 design variables describing materials, gauges, and vehicle shape.
This results in a 124-dimensional optimization problem.

Figure \ref{fig:real} shows that Newton-BO enhances the efficacy of TuRBO and achieves the best performance among all methods on the MOPTA08 problem.
In this problem, the midpoint is not close to the optimal point, resulting in initial strategies of GIBO and ALEBO performing comparably to random search.
ALEBO outperforms the other embedding approaches on the MOPTA08, while
GIBO stagnates and performs worse than random search.
RDUCB performs comparably to TuRBO, despite MOPTA08 lacking the additive structure.

\paragraph{Rover trajectory planning}
We consider a rover trajectory planning problem defined in \cite{ebo}.
The objective is to maximize a 60-dimensional reward function $f$ over the points on the trajectory.

Figure \ref{fig:motion} shows that Newton-BO enhances the efficacy of TuRBO and achieves the best performance among all methods on the rover trajectory planning problem.
Both TuRBO and D-scaled GP-EI show satisfactory performance.
Interestingly, additive models show better performance than embedding-based methods.

\paragraph{Robot pushing.}
We consider a robot pushing problem defined in \cite{ebo}.
The objective is to maximize the reward by pushing two objects with two robot hands, resulting in a noisy 14D control problem.

Figure \ref{fig:motion} shows that Newton-BO has a similar performance to TuRBO while outperforming other methods on the Robot pushing problem.
D-scaled GP-EI continues to show satisfactory performance.
Although the problem is low-dimensional, GIBO stagnates and performs the worst.

\paragraph{MuJoCo Locomotion Tasks}
We consider MuJoCo tasks in reinforcement learning \cite{mujoco}.
The objective is to learn a linear policy that maximizes the accumulative reward.
We focus on the 33-dimensional Hopper task and 102-dimensional Walker2d task.

Figure \ref{fig:motion} shows that Newton-BO achieves the best performance among all methods on the Walker2d problem.
It has a similar performance to TuRBO while outperforming other methods on the Hopper problem.
D-scaled GP-EI is not good at both MuJoCo tasks.
Interestingly, random search outperforms a wide range of BO methods on the Hopper problem but shows the worst performance on the Walker2d problem.

\subsection{Additional Experiments}
Figure \ref{fig:time} compares runtime performance across methods for 50- and 100-dimensional synthetic functions.
The computational cost of BO methods generally arises from fitting a surrogate model and optimizing the acquisition function.
Additionally, embedding-based methods require extra time to compute embeddings and their inverses.
Newton-BO requires an additional $O(D^2)$ time to compute Hessians of GPs.
CMA-ES converges fastest, as it does not require training surrogates or optimizing acquisition functions.
Among BO methods, Newton-BO achieves the best runtime performance on Ackley functions,
while RDUCB achieves the best runtime performance on Griewank functions.

Appendix \ref{sec:ablation} ppresents an ablation study of the key components of Newton-BO to examine how Newton methods and restarting strategies influence convergence.
We compared Newton-BO, D-scaled GP-PES, and Newton-BO with random restarts on 50-dimensional and 100-dimensional synthetic functions. 
The results show that removing either component degrades BO performance, with the absence of Newton methods causing the most notable decline.

\section{Conclusion}
\label{sec:conclusion}
In this paper, we introduce Newton-BO, a novel trust-region BO method that incorporates the derivatives of GPs
for enhancing the sampling efficiency of TuRBO.
This novel scheme is realized by (1) constructing multiple local quadratic models using derivatives from a global GP,
enabling heterogeneous modeling while maintaining the same sample efficiency of a global GP,
and (2) selecting new sample points by solving the bound-constrained quadratic program in multiple trust regions.
Comprehensive experimental evaluations demonstrate that Newton-BO significantly enhances the efficacy of TuRBO 
and outperforms a wide range of high-dimensional BO methods on a set of synthetic functions and real-world applications.
Furthermore, we provide a convergence analysis for our method.

While we mitigate the problem of vanishing derivatives using D-scaled GPs,
we will focus on developing better schemes to address this challenge in the future.

\begin{ack}
This work is supported in part by National Natural Science Foundation of China (62192783), Jiangsu Science and Technology Major Project (BG2024031), the Fundamental Research Funds for the Central Universities (14380128) and the Collaborative Innovation Center of Novel Software Technology and Industrialization.
\end{ack}

\bibliography{0_refs}

\appendix
\newpage
\onecolumn

\section{Proof}
\label{sec:proof}
\begin{algorithm2e}[tbh]
    \KwIn{$\mathbf{s}_k,~\Delta_k,~0<\eta_0\leq\eta_1<1,~0<\beta_1<1<\beta_2,~\mu\geq1$}
    \KwOut{$\Delta_{k+1}$}
    Compute the ratio
    $$ \rho_k := \frac{f(\mathbf{x}_k)-f(\mathbf{x}_k+\mathbf{s}_k)}{m_k(\mathbf{x}_k)-m_k(\mathbf{x}_k+\mathbf{s}_k)}. $$
    \If{$\rho_k\geq\eta_1$}{
        \begin{equation*}
            \Delta_{k+1}\leftarrow\min\{\beta_2\Delta_k, \mu\|\mathbf{g}_k\|_2\}.
        \end{equation*}
    }
    \ElseIf{$\rho_k<\eta_0$}{
    $$\Delta_{k+1}\leftarrow\beta_1\Delta_k.$$
    }
    \Else{
        $\Delta_{k+1}\leftarrow \Delta_k$\;
    }
    \Return{$\Delta_{k+1}$}
\caption{The trust-region update strategy in derivative-free optimization}
\label{algo:TR}
\end{algorithm2e}

\begin{lemma}[Lemma 6 in \cite{conn1997convergence}]\label{lem-model-decrease}
    At every iteration $k$, one has that
    $$m_k(\mathbf{x}_k)-m_k(\mathbf{x}_k+\mathbf{s}_k) \geq 
    \kappa_{mdc}\|\mathbf{g}_k\|\min\left(\Delta_k,\frac{\|\mathbf{g}_k\|}{\kappa_h}\right),$$
    for some constant $\kappa_{mdc}\in(0,1)$ independent of $k$, where $\kappa_h=\max\{\kappa_{fg},\kappa_{fh},\kappa_{mh}\}$.
\end{lemma}

First, we show that the error between the objective and the model decreases at least quadratically with the trust-region radius.

\begin{theorem}\label{th:m-bound}
    Assume that Assumption \ref{a:gp-bound}, \ref{a:f-gh-bound}, and \ref{a:mh-bound} hold. Then given $\delta\in(0,1)$ there is $\kappa_{em}$ such that
    $$\mathbb{P}\left(|f(\mathbf{x}) - m_k(\mathbf{x})| \leq \kappa_{em} \max\{\Delta_k,\Delta_k^2\},~\forall \mathbf{x}\in\mathcal{B}_k~\forall k\right) \geq 1-\delta.$$
\end{theorem}
\begin{proof}
    It follows from Taylor's theorem that
    $$ f(\mathbf{x}_k+\mathbf{s}) = f(\mathbf{x}_k) + \nabla f(\mathbf{x}_k)^\top\mathbf{s}
     + \int_0^1[\nabla f(\mathbf{x}_k+t\mathbf{s})-\nabla f(\mathbf{x}_k)]^\top\mathbf{s}dt, $$
    for some $t\in(0,1)$. Then
\begin{align}
    &|m_k(\mathbf{x}_k+\mathbf{s}) - f(\mathbf{x}_k+\mathbf{s})| \notag\\
    &= \left|[\mathbf{g}_k - \nabla f(\mathbf{x}_k)]^\top\mathbf{s} + \frac12 \mathbf{s}^\top\mathbf{B}_k \mathbf{s}
    - \int_0^1[\nabla f(\mathbf{x}_k+t\mathbf{s})-\nabla f(\mathbf{x}_k)]^\top\mathbf{s}dt\right| \notag\\
    &\leq \left\|\nabla \mu_k(\mathbf{x}_k) - \nabla f(\mathbf{x}_k)\right\|_2 \|\mathbf{s}\|_2  + \|\nabla\sigma_k(\mathbf{x}_k)\|_2 \|\mathbf{s}\|_2 + (\kappa_{mh}/2) \|\mathbf{s}\|_2^2 + \kappa_{fh}\|\mathbf{s}\|_2^2
    \label{eq:em}
\end{align}

It follows from Equation \ref{eq:gp-bound} that
\begin{equation*}
    \mathbb{P}\left( \|\nabla \mu_k(\mathbf{x}_k) - \nabla f(\mathbf{x}_k)\|_2
    \leq \sqrt{\beta(\tau)}\|\nabla\sigma_k(\mathbf{x}_k)\|_2, \forall k \right) \geq 1-\delta.
\end{equation*}
In fact, assume without loss of generality that $f(\mathbf{x}_k) - \mu_t(\mathbf{x}_k) 
\leq \sqrt{\beta(\tau)}\sigma_t(\mathbf{x}_k) + \gamma(\tau)$,
then following the continuity of $f(\mathbf{x}),~\mu_t(\mathbf{x})$ and $\sigma_t(\mathbf{x})$,
there is $\varepsilon\in(0,1)$ such that $\forall i\in\{1,\dots,D\}$
$$f(\mathbf{x}_k+\varepsilon\mathbf{e}_i) - \mu_t(\mathbf{x}_k+\varepsilon\mathbf{e}_i)
\leq \sqrt{\beta(\tau)}\sigma_t(\mathbf{x}_k+\varepsilon\mathbf{e}_i) + \gamma(\tau).$$
Hence, combing the above two inequalities, one has that
$$\frac{f(\mathbf{x}_k+\varepsilon\mathbf{e}_i) - f(\mathbf{x}_k)}{\varepsilon} - \frac{\mu_t(\mathbf{x}_k+\varepsilon\mathbf{e}_i) - \mu_t(\mathbf{x}_k)}{\varepsilon}
\leq \sqrt{\beta(\tau)}\frac{\sigma_t(\mathbf{x}_k+\varepsilon\mathbf{e}_i) - \sigma_t(\mathbf{x}_k)}{\varepsilon}.$$
Letting $\varepsilon\rightarrow0$, one has that
$$\frac{\partial f(\mathbf{x}_k)}{\partial x_i} - \frac{\partial \mu_t(\mathbf{x}_k)}{\partial x_i}
\leq \sqrt{\beta(\tau)}\frac{\partial \sigma_t(\mathbf{x}_k)}{\partial x_i}.$$
Similarly, if $\mu_t(\mathbf{x}_k) - f(\mathbf{x}_k)
\leq \sqrt{\beta(\tau)}\sigma_t(\mathbf{x}_k) + \gamma(\tau)$, then
$$\frac{\partial \mu_t(\mathbf{x}_k)}{\partial x_i} - \frac{\partial f(\mathbf{x}_k)}{\partial x_i}
\leq \sqrt{\beta(\tau)}\frac{\partial \sigma_t(\mathbf{x}_k)}{\partial x_i},~ \forall i\in\{1\dots D\}.$$
Since then, it has been proved the event
$|f(\mathbf{x}_k) - \mu_t(\mathbf{x}_k)| 
\leq \sqrt{\beta(\tau)}\sigma_t(\mathbf{x}_k) + \gamma(\tau)$ implies that
$\|\nabla \mu_k(\mathbf{x}_k) - \nabla f(\mathbf{x}_k)\|_2 
\leq \sqrt{\beta(\tau)} \|\nabla\sigma_k(\mathbf{x}_k)\|_2$.

Since $\sigma_k$ admits a modulus of continuity according to Equation \ref{eq:gp-Lipschitz},
there is $\kappa_{eg}$ such that $\|\nabla\sigma_k(\mathbf{x}_k)\|_2\leq\kappa_{eg}\Delta_k$. Then
\begin{equation}\label{eq:eg}
    \mathbb{P}\left( \|\nabla \mu_k(\mathbf{x}_k) - \nabla f(\mathbf{x}_k)\|_2 
    \leq \kappa_{eg} \sqrt{\beta(\tau)} \Delta_k, \forall k \right) \geq 1-\delta.
\end{equation}
Combining Equation \ref{eq:em} and \ref{eq:eg}, one has that
\begin{equation*}
    \mathbb{P}\left[|m_k(\mathbf{x}_k+\mathbf{s}) - f(\mathbf{x}_k+\mathbf{s})| 
    \leq (\kappa_{eg} \sqrt{\beta(\tau)} + \kappa_{eg} + \kappa_{mh}/2 + \kappa_{fh}) \max\{\Delta_k,\Delta_k^2\},~\forall k \right] \geq 1-\delta
\end{equation*}
Hence, $\kappa_{em} = \kappa_{eg} \sqrt{\beta(\tau)} + \kappa_{eg} + \kappa_{mh}/2 + \kappa_{fh}$.
\end{proof}

Next, we show that an iteration must be successful if the current iterate is not critical and the trust-region radius is small enough.

\begin{lemma}\label{le:radius-lb}
Assume that Assumption \ref{a:gp-bound}-\ref{a:mh-bound} hold. In addition,
assume that there is a constant $\kappa_g>0$ such that $\|g_k\| \geq \kappa_g$ for all $k$.
Then given $\delta\in(0,1)$ there is a constant $\kappa_d$ such that
$$\mathbb{P}\left(\Delta_k > \kappa_d, ~\forall k\right) \geq 1-\delta.$$
\end{lemma}
\begin{proof}
    It follows from Lemma 7 in \cite{conn1997convergence} that if 
    $|f(\mathbf{x}) - m_k(\mathbf{x})| \leq \kappa_{em} \max\{\Delta_k,\Delta_k^2\}$, 
    then $\forall k,~\Delta_k > \kappa_d$, where
$$ \kappa_d = \beta_1\min\left(1, \frac{\kappa_{mdc}\kappa_g(1-\eta_1)}{\max(\kappa_h,\kappa_{em})}\right).$$
And since it follows from Theorem \ref{th:m-bound} that
$$\mathbb{P}\left(|f(\mathbf{x}) - m_k(\mathbf{x})| \leq \kappa_{em} \max\{\Delta_k,\Delta_k^2\},~\forall \mathbf{x}\in\mathcal{B}_k~\forall k\right) \geq 1-\delta.$$
and hence, we obtain
$$\mathbb{P}\left(\Delta_k > \kappa_d,~\forall k\right) \geq 1-\delta.$$

\end{proof}

This property guarantees that the radius is unlikely to become too small as long as the gradient of the GP does not vanish.
We then analyze the criticality of the limit point of the sequence of iterates.

\begin{theorem}\label{th:g0}
    Assume that Assumption \ref{a:gp-bound}-\ref{a:mh-bound} hold. Then it holds that
    $$\liminf_{k\rightarrow\infty}\|\mathbf{g}_k\|_2=0$$
\end{theorem}
\begin{proof}
    We proceed by contradiction. Suppose there is $\kappa_g>0$ such that $\|\mathbf{g}_k\|\geq\kappa_g$ for all $k$.
    It follows from Theorem 9 in \cite{conn1997convergence} that if $\Delta_k > \kappa_d$ for all $k$, 
    then
    $$ f(\mathbf{x}_0)-f(\mathbf{x}_{k+1}) \geq \frac12\sigma_k\kappa_g\eta_0\min\left(\frac{\kappa_g}{\kappa_h},\kappa_d\right) $$
    where $\sigma_k$ is the number of successful iterations up to iteration $k$. In our case,
    it follows from Lemma \ref{le:radius-lb} that
    $$\mathbb{P}\left(\Delta_k > \kappa_d, \forall k\right) \geq 1-\delta.$$
This implies that
$$ \mathbb{P}\left(f(\mathbf{x}_0)-f(\mathbf{x}_{k+1}) 
\geq \frac12\sigma_k\kappa_g\eta_0\min\left(\frac{\kappa_g}{\kappa_h},\kappa_d\right)\right) \geq 1-\delta.$$
And since $\lim_{k\rightarrow\infty}\sigma_k=+\infty$, one has that $\forall M\in\mathbb{R}~\exists k$,
$$ \mathbb{P}\left(f(\mathbf{x}_0)-f(\mathbf{x}_{k+1}) > M\right) \geq 1-\delta,$$
which contradicts the fact that $f$ is bounded.
\end{proof}

\begin{lemma}\label{le:f0}
    Assume that Assumption \ref{a:gp-bound}-\ref{a:mh-bound} hold.
    If there is a subsequence $\{k_i\}$ such that $\lim_{i\rightarrow\infty}\|\mathbf{g}_{k_i}\|=0$,\
    then given $\delta\in(0,1)$ it holds that $\forall \epsilon\in(0,1)~\exists N$,
    \begin{equation*}
        \mathbb{P}\left( \|\nabla f(\mathbf{x}_{k_i})\|_2 < \epsilon,~\forall i>N \right) \geq 1-\delta.
    \end{equation*}
\end{lemma}
\begin{proof}
    It follows from Equation \ref{eq:eg} that
\begin{equation*}
    \mathbb{P}\left( \|\nabla f(\mathbf{x}_{k_i}) - \mathbf{g}_{k_i}\|_2
    \leq \kappa_{eg} \sqrt{\beta(\tau)} \Delta_{k_i},~\forall i \right) \geq 1-\delta.
\end{equation*}
And since $\Delta_{k_i}\leq\mu\|\mathbf{g}_{k_i}\|_2$ (according to Algo. \ref{algo:TR}), one has that
\begin{equation*}
    \mathbb{P}\left( \|\nabla f(\mathbf{x}_{k_i}) - \mathbf{g}_{k_i}\|_2
    \leq \kappa_{eg} \sqrt{\beta(\tau)} \mu\|\mathbf{g}_{k_i}\|_2,~\forall i \right) \geq 1-\delta.
\end{equation*}
And since $\|\nabla f(\mathbf{x}_{k_i})\|_2 \leq \|\mathbf{g}_{k_i}\|_2 + \|\nabla f(\mathbf{x}_{k_i}) - \mathbf{g}_{k_i}\|_2$,
one has that
\begin{equation*}
    \mathbb{P}\left( \|\nabla f(\mathbf{x}_{k_i})\|_2
    \leq (1+\kappa_{eg} \sqrt{\beta(\tau)} \mu)\|\mathbf{g}_{k_i}\|_2,~\forall i \right) \geq 1-\delta.
\end{equation*}
Combining the limit $\lim_{i\rightarrow\infty}\|\mathbf{g}_{k_i}\|_2=0$ and the above equation, one has that
$\forall \epsilon\in(0,1)~\exists N$,
    \begin{equation*}
        \mathbb{P}\left( \|\nabla f(\mathbf{x}_{k_i})\|_2 < \epsilon, \forall i>N \right) \geq 1-\delta.
    \end{equation*}
\end{proof}

Finally, the local convergence result of Theorem \ref{th:converge} immediately follows from Theorem \ref{th:g0} and Lemma \ref{le:f0}.

\section{Additional Experiments}
\subsection{Ablation Study}
\label{sec:ablation}

We conducted an ablation study to understand how Newton methods and restarting strategies influence convergence. 
We compared Newton-BO, D-scaled GP-PES, and Newton-BO with random restarts on 50-dimensional and 100-dimensional synthetic functions. 
The results are presented in Figure \ref{fig:ablation}.
Removing either component reduced BO performance, with the absence of Newton methods causing the most notable decline.
\begin{figure*}[!htb]
    \centering
    \includegraphics[width=\linewidth]{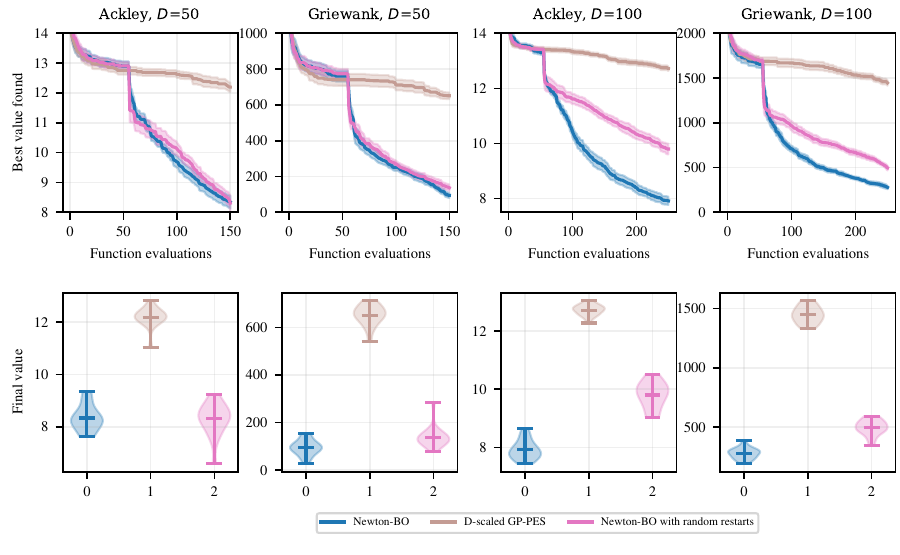}
\caption{Ablation study results comparing Newton-BO and D-scaled GP-PES.}
\label{fig:ablation}
\end{figure*}

\section{Method Implementation and Experimental Setting}
\label{sec:settings}
\paragraph{Hardware Usage.}
The synthetic and real-world experiments were conducted on a laptop equipped with a 2.9 GHz 8-Core Intel Core i7-7820HK CPU and 32 GB of RAM.

\paragraph{Baselines.}
For ALEBO, we utilize the implementations in Adaptive Experimentation Platform (Ax\footnote{https://github.com/facebook/Ax}).
Implementations for Sobol, D-scaled GP-EI, and D-scaled GP-PES are sourced from BoTorch\footnote{https://botorch.org/}.
The implementation for CMA-ES is derived from the package\footnote{https://github.com/CMA-ES/pycma}.
Implementation for TuRBO\footnote{https://github.com/uber-research/TuRBO}, GIBO\footnote{https://github.com/sarmueller/gibo}, and RDUCB\footnote{https://github.com/huawei-noah/HEBO/tree/master/RDUCB} are derived from their respective original papers.
We implement REMBO, SIR-BO, KSIR-BO, and Add-GP-UCB in Python following MATLAB codes provided in the original papers.

\paragraph{Newton-BO details}
We configure the following hyperparameters for Newton-BO in all experiments: $\tau_\text{succ}=\tau_\text{fail}=3,\ \Delta_\text{max}=0.8,$ and $\Delta_\text{init}=0.4,$
which resemble the settings used in TuRBO.
Besides, $\Delta_\text{min}$ is employed to balance local search and global search. 
To accelerate global convergence, we vary $\Delta_\text{min}$ by problem type:
(i) for synthetic functions and the MOPTA08 problem, $\Delta_\text{min}=0.05$,
(ii) for the Lasso-synt\_high, Rover, and Hopper problems, $\Delta_\text{min}=0.1$,
(iii) for Lasso-DNA, Robot pushing and Walker2d problems, $\Delta_\text{min}=0.2$.

Following TuRBO, if a new point improves the current best solution we increment the success counter and reset the failure counter to zero,
otherwise we set the success counter to zero and increment the failure counter.

For each trust region, we initialize $\Delta\leftarrow\Delta_\text{init}$ and terminate the trust region when $\Delta < \Delta_\text{min}$ or $\|\mathbf{g}\|\leq 10^{-5}$.
A new restart is then selected by maximizing the PES function.

\end{document}